\documentclass{article}

\usepackage{microtype}
\usepackage{graphicx}
\usepackage{subfigure}
\usepackage{booktabs} 
\usepackage[table]{xcolor}
\usepackage{tikz}
\usepackage{adjustbox}

\usepackage{natbib}
\definecolor{lightergray}{rgb}{0.9, 0.9, 0.9}
\definecolor{evenlightergray}{rgb}{0.95, 0.95, 0.95}
\usepackage{hyperref}

\usepackage{boldline}
\usepackage{wrapfig}

\usepackage[final]{neurips_2024}

\usepackage{amsmath}
\usepackage{amssymb}
\usepackage{mathtools}
\usepackage{amsthm}
\usepackage{thm-restate}
\usepackage{algorithmic}

\newtheoremstyle{mystyle} 
  {10pt}  
  {7pt}   
  {}      
  {}      
  {\bfseries} 
  {.}     
  {.5em}  
  {}      
  
\theoremstyle{mystyle} 

\usepackage[capitalize,noabbrev]{cleveref}

\usepackage{graphicx}
\usepackage{float}
\usepackage{caption}
\usepackage{multirow, multicol}
\usepackage{tabularx}
\usepackage{bbm}
\usepackage{graphicx}
\usepackage[font={small,it}]{caption}

\makeatletter
\allowdisplaybreaks
\makeatother

\usepackage{hyperref}
\usepackage{url}

\usepackage{enumitem}

\usepackage{bbm}

\usepackage{comment}

\usepackage{latexsym,graphicx}
\usepackage{amsthm,amsmath,amssymb,enumitem}
\usepackage{empheq}
\usepackage{xspace}
\usepackage{bm}
\usepackage{ifpdf}
\usepackage{xcolor}
\usepackage{algorithm}
\usepackage{nicefrac}
\usepackage{wrapfig}

\allowdisplaybreaks

\definecolor{Darkblue}{rgb}{0,0,0.4}
\definecolor{Brown}{cmyk}{0,0.81,1.,0.60}
\definecolor{Purple}{cmyk}{0.45,0.86,0,0}
\usepackage{hyperref}
\hypersetup{colorlinks=true,
           citebordercolor={.6 .6 .6},linkbordercolor={.6 .6 .6},%
citecolor=Darkblue,urlcolor=black,linkcolor=red}
\newcommand{\lref}[2][]{\hyperref[#2]{#1~\ref*{#2}}}
\usepackage{cleveref}

\newtheorem{theorem}{Theorem}[section]

\newtheorem{corollary}[theorem]{Corollary}

\numberwithin{algorithm}{section}

\newcommand{\junk}[1]{}
\newcommand{\ignore}[1]{}

\newcommand{\argmin}{\operatorname{argmin}}

\newcounter{note}[section]

\newcommand{\qedsymb}{\hfill{\rule{2mm}{2mm}}}

\newcommand{\initOneLiners}{%
    \setlength{\itemsep}{0pt}
    \setlength{\parsep }{0pt}
    \setlength{\topsep }{0pt}
}

\newcommand{\squishlist}{
 \begin{list}{$\bullet$}
  { \setlength{\itemsep}{0pt}
     \setlength{\parsep}{3pt}
     \setlength{\topsep}{3pt}
     \setlength{\partopsep}{0pt}
     \setlength{\leftmargin}{1.5em}
     \setlength{\labelwidth}{1em}
     \setlength{\labelsep}{0.5em} } }

\newcommand{\squishend}{
  \end{list}  }

\DeclarePairedDelimiterX{\infdivx}[2]{(}{)}{%
  #1\;\delimsize\|\;#2%
}

\setlist[enumerate]{itemsep=12.0pt,parsep=1.0pt,before={\parskip=10.0pt},leftmargin=0.75cm}

\title{Majority Kernels: An Approach to Leverage Big Model Dynamics for Efficient Small Model Training}

\usepackage{thmtools} 
\usepackage{thm-restate}
\usepackage{amsthm}

\usepackage{ulem}
\usepackage{color}

\newcommand{\wm}{\langle \widetilde{w} \rangle}
\newcommand{\wmp}{\langle \widetilde{w} \rangle_p}
\newcommand{\wmpi}[1]{\langle \widetilde{w}^{(#1)} \rangle_{p_{#1}}}
\newcommand{\we}{\widetilde{w}}

\newcommand{\thetae}{\tilde{\theta}}

\newcommand{\thetamp}{\bar{\theta}_p}
\newcommand{\thetam}{\bar{\theta}}

\author{
Hanna Mazzawi \\ Google Research, NY \\ mazzawi@google.com  \And  
Pranjal Awasthi$^1$ \\ Google Research, CA \\ pranjalawasthi@google.com \And
Xavi Gonzalvo$^1$ \\ Google Research, NY \\ xavigonzalvo@google.com \And 
Srikumar Ramalingam \\ Google Research, NY \\ rsrikumar@google.com 
}

\begin{document}

\setlength{\abovedisplayskip}{10pt} 

\maketitle
\addtocounter{footnote}{1}

\footnotetext{Equal contribution.}
\begin{abstract}

Recent breakthroughs and successful deployment of large language and vision models in a constrained environment predominantly follow a two phase approach. First, large models are trained to achieve peak performance, followed by a model shrinking method to meet hardware constraints; Methods like distillation, compression or quantization help leverage the highly performant large models to induce smaller performant ones. Formally, this can be seen as the problem of identifying an optimal model of size $n$ from a larger model of size $k \cdot n$, where $k > 1$ is the overparameterization factor. This paper explores the hypothesis that a single training run can simultaneously train a larger model for performance and derive a smaller model for deployment.

Our contribution is an effective architectural change, namely, {\it Majority Kernels} that is compatible with the main standard architectures such as multi-layer perceptrons (MLPs), Residual networks (ResNets), and Transformers. We demonstrate that applying our technique can modify the training dynamics resulting in performance gains across architectures and tasks while maintaining the inference performance consistent. Furthermore, our approach adds minimal overhead to the cost incurred (wall clock time) at training time. The proposed approach shows strong performance on a wide variety of datasets and models, even outperforming strong baselines such as distilled ensembles as well as combinatorial optimization methods based on submodular optimization.

\end{abstract}

\section{Introduction}


Overparametrized models have become a cornerstone in advancing deep learning, particularly when trained with first-order methods. The empirical evidence suggests that such models often demonstrate superior performance, a trend that persists without evident saturation points, assuming that data of sufficient quality and quantity is available~\citep{scaling_laws20}. The benefits extend beyond mere accuracy improvements; these large models enhance optimization stability and enable more robust generalization in diverse scenarios. This approach aligns with ongoing efforts to push the boundaries of deep learning capabilities through continued growth in model size.
This shift towards larger models is evident from the early influential works in computer vision with CNNs~\citep{krizhevsky2012imagenet,szegedy2015going} and ResNets~\citep{he2016deep} to the transformer architectures in language modeling~\citep{vaswani2017attention, brown2020language, chowdhery2022palm, chen2022pali}.




While larger models are pivotal for achieving peak performance in deep learning, their practical deployment, especially on resource-constrained devices like smartphones, necessitates consideration of the model's footprint. These devices impose limits on the number of parameters and the computational cost. This can be effectively addressed by a two-step process. First, train an overparameterized model; then, employ post-training techniques to compress and tailor the model to meet specific operational constraints. This strategy allows for the benefits of large models to be realized in environments with stringent resource limitations \citep{onceforall19}.
%
%
%
%
%
%
Techniques like model pruning and quantization are designed to create models with reduced memory needs, maintaining performance while fitting into more restricted environments~\citep{lecun1989optimal, han2015deep, frankle2018lottery, cai2020zeroq, nagel2020up}.
Similarly, model distillation focuses on training smaller models that can faithfully approximate a larger model (typically of the same architecture) \citep{Hinton06, buci2006model}. This approach of post-training optimization plays a vital role in making advanced models viable for everyday applications.


This paper introduces a novel concept, questioning whether it is possible to streamline the above two stage process into a single training run. We explore the feasibility of simultaneously conducting larger-scale training while also producing a smaller, immediately deployable model for inference. This concept involves increasing the size of a compact model in such a way that it incurs minimal additional training overhead, yet reverts to its original, smaller size for inference. 

A common method to increase the model size to achieve peak performance is ensembling (see Appendix~\ref{sec:related} for more related work). Given a base model \( f(x; \theta) \) with \( \theta \in \mathbb{R}^n \), an effective way to improve the model is by training $k$ copies independently and average their predictions. 
For the ensembling method, some techniques aim to mitigate the use of extra compute and memory (without needing a second stage to reduce the size),
\begin{itemize}
    \item Snapshot ensembling \cite{huang2017snapshot} removes the need to train $k$ independent copies of the base model $f(x;\theta)$, reducing the training compute to be somewhat comparable with the original model training requirements. However, this algorithm does not solve the increase in compute and memory for inference. 

    \item Mixture of experts \citep{shazeer2017outrageously}, An MOE creates a model with $k$ experts each structurally equivalent to $f(x;\theta)$. This effectively creates an overparameterized model with number of parameters approximately being 
    $k\cdot n$, and comparable compute. However, while computation cost per inference remains similar to the original model, memory requirements increase due to the larger parameter count.
    
    \item Bayesian neural networks (BNN) \citep{magris2023bayesian} can be seen as having multiple models \( f(x; \theta) \), where \( \theta \) is sampled from a posterior distribution \( p(\theta|\mathcal{D}) \) given the data~\( \mathcal{D} \). During training and inference, we draw samples of \( \theta \) to capture the uncertainty and variability in the model predictions. However, inference requires multiple runs of the model, exceeding the original compute cost. Additionally, while there is recent work about memory footprint efficiency~\citep{dusenberry20bnn}, regular BNNs have to store parameters and their uncertainties which can make the dimension of $\theta$ be larger than $2n$.
\end{itemize}


    

In this work we make the following contributions. 
\begin{itemize}
    \item We present \textit{Majority Kernels} (MK), a novel algorithm that increases the size of training models similarly to ensembling. When applied to a base model, this approach expands its parameters while maintaining the same inference compute and memory requirements as the original model, with only a minimal increase in training computation.
    \item We theoretically analyze the proposed algorithm through the lenses of implicit regularization. We demonstrate how the training dynamic change when applying our technique (in this context, training dynamics from an optimization perspective, involve analyzing how the model's parameters converge towards optimal values over time, influenced by factors such as learning rate, gradient behavior, and the chosen optimization algorithm).
    This helps us understand how the algorithm not only naturally limits the complexity of the model class but also seeks out more stable and generalizable solutions by focusing on flatter regions of the loss landscape. 
    \item We present an extended empirical analysis to explore the efficacy of our algorithm across different architectures and datasets, showcasing their role in facilitating implicit overparameterized training. We compare our algorithm with strong baselines such as
    \begin{itemize}
        \item Distilled ensembles - a baseline that requires orders of magnitude more compute.
        
        \item Combinatorial optimization methods - a baseline that tries to streamline the above mentioned two phases when the increase in capacity is done via the model's dimension.
    \end{itemize}
\end{itemize} 
Our algorithm is remarkably simple yet effective; It is particularly suitable for large language models due to its negligible compute overhead during training and its lack of impact on inference compute and memory.

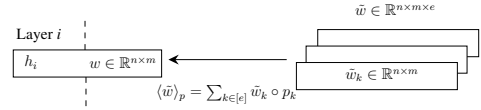
\begin{wrapfigure}{r}{0.45\textwidth}
    \centering
    \begin{adjustbox}{max width=1\linewidth}
    \tikzset{every picture/.style={line width=0.75pt},every node/.style={font=\large}} 

\begin{tikzpicture}[x=0.75pt,y=0.75pt,yscale=-1,xscale=1]

\draw  [fill={rgb, 255:red, 255; green, 255; blue, 255 }  ,fill opacity=1 ] (325,50) -- (491,50) -- (491,78) -- (325,78) -- cycle ;
\draw  [fill={rgb, 255:red, 255; green, 255; blue, 255 }  ,fill opacity=1 ] (313,67) -- (479,67) -- (479,95) -- (313,95) -- cycle ;
\draw   (11,71) -- (164,71) -- (164,99) -- (11,99) -- cycle ;
\draw  [fill={rgb, 255:red, 255; green, 255; blue, 255 }  ,fill opacity=1 ] (303,84) -- (469,84) -- (469,112) -- (303,112) -- cycle ;
\draw    (293,81) -- (174,81) ;
\draw [shift={(171,81)}, rotate = 360] [fill={rgb, 255:red, 0; green, 0; blue, 0 }  ][line width=0.08]  [draw opacity=0] (10.72,-5.15) -- (0,0) -- (10.72,5.15) -- (7.12,0) -- cycle    ;
\draw  [dash pattern={on 4.5pt off 4.5pt}]  (85,41) -- (85,71) ;
\draw  [dash pattern={on 4.5pt off 4.5pt}]  (86,99) -- (86,129) ;

\draw (21,75.4) node [anchor=north west][inner sep=0.75pt]    {$h_{i}$};
\draw (88,76.4) node [anchor=north west][inner sep=0.75pt]    {$w\in \mathbb{R}^{n\times m}$};
\draw (352,87.4) node [anchor=north west][inner sep=0.75pt]    {$\tilde{w}_{k} \in \mathbb{R}^{n\times m}$};
\draw (362,21.4) node [anchor=north west][inner sep=0.75pt]    {$\tilde{w} \in \mathbb{R}^{n\times m\times e}$};
\draw (157,105.4) node [anchor=north west][inner sep=0.75pt]    {$\langle \tilde{w} \rangle _{p} =\sum _{k\in [e]}\tilde{w}_{k} \circ p_{k}$};
\draw (13,48) node [anchor=north west][inner sep=0.75pt]   [align=left] {Layer \textit{i}};

\end{tikzpicture}
    \end{adjustbox}
    \caption{The majority kernels for the $i$-th layer.}
    \label{fig:mk_image}
\end{wrapfigure}

\section{The Majority Kernels Algorithm}
In this section we introduce the Majority Kernels (MK) algorithm.
This approach involves training each layer of a Deep Neural Network (DNN) with an expanded version of their internal kernels. During training, MK aggregates these expanded kernels by randomly averaging extra parameters into the layer's original dimensions. At the inference stage, the kernel reverts to the average of these expanded versions.

It is well understood that ensembling of models often produces a model that is superior and more robust as compared to the base models \citep{huang2017snapshot, fort1912deep}.
Consider a model \( f: {\cal X} \rightarrow \mathbb{R}^k \), where \( {\cal X} \subseteq \mathbb{R}^d \) represents the domain of the data distribution. Define \( f \) as an ensemble of \( e \) different models: \( f(x;\theta) = \frac{1}{e} \sum_{i=1}^e f_i(x;\theta_i) \) for \( x \in {\cal X} \). Although typically ensembling is done in the model (function) space, if one could do ensembling in the parameter space itself then $f(x;\theta)$ would correspond to a powerful model that is also small in size. However, naive ensembling in the parameter space often performs poorly as the parameters of the different models might not be aligned along the same local optima. Our algorithm maintains parameter alignment by using stochastic weighted averaging throughout training, while using using average for inference. More formally,




Assume a multilayered model where each layer is defined as $h: \mathbb{R}^n \to \mathbb{R}^m$:
$$
h(x) = \phi(x w + b),
$$
where $w\in \mathbb{R}^{n\times m}$, $b\in\mathbb{R}^m$, and $\phi$ is a non-linear activation function like ReLU.

MK maintains the dimensionality of $h: \mathbb{R}^n \to \mathbb{R}^m$ but it uses an extended kernel $\we \in \mathbb{R}^{n\times m \times e}$ that allows the learning over an order of magnitude larger than the original kernel $w$ (see Figure~\ref{fig:mk_image}). The way the extended kernel $\we$ is used is as follows,
\begin{equation}
  \label{eq:w}
h(x) = \phi(x \wmp + b) \quad s.t. \quad \wmp = \sum_{k\in[e]} \we_{k} \circ p_{k},
\end{equation}
where $\circ$ indicates pointwise multiplication, $\we_k \in \mathbb{R}^{n\times m}$ refers to the $k$-th extended dimension of $\we$ and $p \in\mathbb{R}^{n\times m \times e}$ is the probability matrix used during training. The probability matrix is constructed as follows: For each pair \( (i, j) \), consider the vector \( (p_{ij1}, \ldots, p_{ije}) \) generated by drawing each component of \( p_{ijk} \) for $k\in [e]$ independently from an exponential distribution. We then normalize such that the sum of components equal 1 and set:
\begin{equation}
  \label{eq:normalized_p}
  (p_{ij1}, \ldots, p_{ije}) \gets \frac{1}{\sum_{k=1}^{e} p_{ijk}} (p_{ij1}, \ldots, p_{ije}).
\end{equation}

This approach guarantees that the extra parameters are leveraged during training by exploring a ball around the mean, while during inference, the theoretical mean is used. In that case (uniform \( p_{i,j} = \nicefrac{1}{e} \cdot \mathbf{1} \))\footnote{Note that $\mathbf{1}$ is a matrix with 1's in all positions.}, and for simplicity $\wm = \frac{1}{e} \sum_{k\in[e]} \we_{k}$.

\begin{algorithm}[h]
\caption{Majority kernels training algorithm.}
\begin{algorithmic}
\STATE For every layer $l$:  Initialize $\we\in\mathbb{R}^{n\times m \times e}$, $b\in\mathbb{R}^{m}$
\WHILE{$s < \mathrm{max\_steps}$}
  \STATE $B \gets $\texttt{NewBatch}()
  \STATE $p \gets \texttt{NormalizedExponentialRandom}()$ \algorithmiccomment{ Equation~\ref{eq:normalized_p}}
  \STATE \texttt{ForwardPass}($\langle \we \rangle$,$p$,$b$,$B$)
  \STATE $\we$, $b \gets \texttt{BackwardPass}(B)$
  \STATE $s+=1$
\ENDWHILE
\STATE return $\langle \we \rangle$, $b$ \algorithmiccomment{Return trained parameters for inference}
\end{algorithmic}\label{main_algo}
\end{algorithm}

Algorithm~\ref{main_algo} describes the training procedure for an architecture agnostic MK approach. The core idea is to use a stochastic version of the extended kernel $\langle \we \rangle_p$ during training, and a final version for inference using a uniform $p$ (i.e., using $\langle \we \rangle$).
\begin{table}
\begin{small}
\hspace{-18pt}\begin{tabular}{c|c|c|c|c}\toprule
    \textbf{Experiment type} & \textbf{Improvement} & \textbf{Vanilla cost} & \textbf{MK cost} & \textbf{Ensemble cost} \\\midrule
    Fully Connected (Cifar10) & +2.07\% in accuracy & 1.2 CPU & 1.4 CPU (+16.67\%) & 3.6 CPU (+300\%)\\
    Convolutional (Imagenet) & +1.01\% in accuracy & 11.15 TPU  &12.18 TPU (+9.24\%)& 33.45 TPU (+300\%) \\
    Transformers (T5) & +0.75\% on avg. Glue & 5.9 TPU & 5.95 TPU (+0.85\%) & 17.7 TPU (+300\%) \\ \bottomrule
\end{tabular}
\end{small}

\vspace{5pt}
\caption{This table summarizes the relative performance improvements (averaged over various settings) for the various domains or architectures compared to vanilla training. Additionally, it presents the cost (in hours) of our training algorithm (with expansion factor 3, i.e., the size of the model is 3 times larger) compared to regular training and compared to the cost of vanilla ensemble of size 3. In Section~\ref{sec:experiments}, we present full detailed results of the above and results against comparing with strong baselines.}
\vspace{-10pt}
\end{table}

\section{Implicit Gradient Regularization}
BEA has been used in many settings to uncover inductive training biases of various optimizers (e.g. gradient descent \cite{barrett2020implicit}, SGD \cite{smith2021on}; momentum \cite{ghosh2023implicit}; Adam and RMSProp \cite{cattaneo2023implicit_bias_adam}), various architectures (e.g., GANs \cite{rosca2021discretization}; diffusion models \cite{gao2023diffusion_bea}), and various training settings like continual-learning \cite{dherin2023implicit}, distributed and federated learning \cite{barba2021federated} and progressively growing networks \cite{deepfusion2024}.

As an example, in the paper \citep{barrett2020implicit}, using Backward Error Analysis (BEA), have demonstrated that for any loss function that is sufficiently differentiable ($L$), the process of gradient descent actually follows an adjusted loss surface, represented as $\tilde{L}$. This adjusted loss surface during training is defined by the as $\tilde{L}(\theta) = L(\theta) + \frac{\ell}{4} \lVert \nabla_\theta L(\theta) \rVert^2$ where $\ell$ is the learning rate.

In this equation, the additional term serves as a regularization component, promoting parameter points where the gradient is low, potentially indicating flatter minima.

In this section we analyze our algorithm using BEA to uncover inductive training biases. For MK, the parameters of the model are defined as $\thetamp=(\wmpi{1}, \ldots, \wmpi{L})$ and $\thetam$ when using uniform element-wise $p=\nicefrac{1}{e}$.
As in \citep{barrett2020implicit}, we define a vector field, in this case on the extended set of parameters $f(\thetamp)=-\nabla L(\thetamp)$ and $f(\thetam)=-\nabla L(\thetam)$. Our goal is to uncover the modified loss function in which the gradient decent algorithm follows to examine its biases.


\begin{restatable}[Backward Error Analysis]{theorem}{bea}
\label{th:bea}
Let $L$ be a sufficiently differentiable function on the parameter space $\theta \in \mathbb{R}^n$. The modified loss when using Majority Kernels is,
\begin{equation*}
\tilde{L}_{MK}(\thetam) = L(\thetam) + \frac{\ell}{4} \lVert \nabla L(\thetam) + \nabla^2 L(\thetam)\cdot\epsilon \rVert^2 + \nabla L(\thetam) \cdot \epsilon.
\end{equation*}
%
%
where $\epsilon \in \mathbb{R}^n$ is 
is the random perturbation of the virtual parameters.
\end{restatable}
See Appendix \ref{sec:proof-main-app} for the proof of the theorem. Furthermore, in Appendix \ref{section:theory}, we further analyze our algorithm using the sharpness aware minimization, and the PAC frameworks. Additionally, this appendix proves Lemma~\ref{lemma:uniform_fallback} showing that without the random probabilities in the algorithm, we fall back on the small model's training dynamics, and lose the leverage of the extra parameters.

\section{Experiments}\label{sec:experiments}
In this section we present empirical results comparing our proposed method against a variety of baselines across various architectures (see the Appendix for more experimental results, including experiments on fully connected networks).

\begin{figure}[htb]
    \centering
    \hspace*{-1.7cm}\includegraphics[scale=0.21]{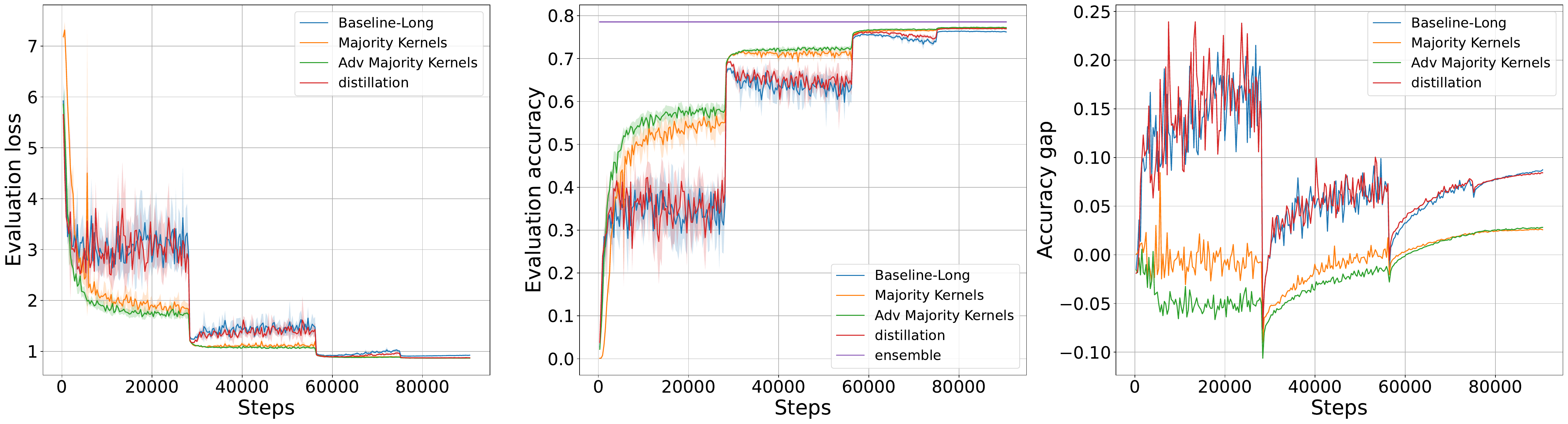}
    \caption{Test curves for \textsc{Baseline-Long} (blue), \textsc{Majority Kernels} (orange), \textsc{Adv-Majority Kernels} (green), \textsc{distilled-Baseline} (red) and between (purple). It is easy to see that the generalization gap (accuracy gap between train and test above) is better in our optimization throughout training with better final performance.\vspace{0.5cm}}
    \label{fig:sam_vs_majority} \vspace{-5pt}
\end{figure}
\begin{figure}[ht]
\begin{minipage}{1\textwidth}
    \centering
    \begin{tabular}{c|c}
    \toprule
         \textbf{Algorithm} & \textbf{Test accuracy} \\ \midrule
         \textsc{Baseline}  & $76.1$  \\
         \textsc{Baseline-long}  & $76.43 \pm 0.076$ \\
         \textsc{Majority Kernels}  & $\textbf{77.20}\pm 0.068$ \\
         \textsc{Adv-Majority Kernels} & $\textbf{~77.33}\pm 0.11$  \\
         \textsc{distilled-Baseline} & ~$77.05 \pm 0.080$  \\ \midrule
         \textsc{ensemble-Baseline} & $78.56\pm 0.041$  \\ \bottomrule
    \end{tabular}
    \vspace{5pt}
    \caption{Results on ImageNet for the various algorithms.}
    \label{tab:cnn_algorithm}
\end{minipage} \vspace{-15pt}
\end{figure}

\subsection{Convolutional Networks}
In this section, we present our results for running our experiments on Imagenet. We experiment with Imagenet on a ResNet50 architecture~\cite{he2016deep}.

In our experiment, we will compare the following algorithms,
\begin{itemize}
    \item \textbf{\textsc{Baseline}}. Training of ResNet50 based on the recipe in~\cite{he2016deep}.
    Training for~90 epochs with batches of size 256, SGD with momentum as an optimizer. Our base learning rate is~0.1 and we have step decay of 0.1 every 30 epochs, we use weight decay of~0.0001. \\We use the standard data augmentation for ImageNet while training: we crop a random segment of the image, and scale it to standard input size of $224\times224$, along with random horizontal flipping of the images. 
    \item \textbf{\textsc{Baseline-long}}. Similar to \textsc{Baseline} but trained for longer (just like majority kernels), that is, trained for~330 epochs, with learning rate drop at epochs 90, 180 and~240.
    \item \textbf{\textsc{Majority Kernels}}. The MK algorithm is implemented on the Baseline model. Due to MK's requirement for more steps to converge, we adopt the \textsc{Baseline-long} configuration. Our expansion factor is $e=3$.
    \item \textbf{\textsc{Adv-Majority Kernels}}. A modification of MK (with expansion factor $e=3$), where an adversarial element is injected into the random probability at each training step. Details in Appendix~\ref{sec:adv-majority}.
    \item \textbf{\textsc{ensemble-Baseline}}. Evaluating true ebsemble performance by training three different \textsc{Baseline-long} models and ensembling them.
    \item \textbf{\textsc{distilled-Baseline}}. \textsc{Baseline-long} includes an additional knowledge distillation loss during training, where
    \textsc{ensemble-Baseline} serves as the teacher.
\end{itemize}  
%
Table~\ref{tab:cnn_algorithm} summarizes the results, showing that our MK algorithm outperforms others in performance. Notably, MK exhibits the most effective balance between training and test performance, highlighting the benefits of its additional regularization and its implicit stochastic sharpness optimization behavior (Section \ref{section:theory}). These effects are evident in the train and test accuracy curves during training, as depicted in Figure~\ref{fig:sam_vs_majority}. Finally, our algorithm increases the training time on a TPU 4x4 by 9.24\% compared to at least 400\% increase of compute for the distilled baseline.
%

\subsection{Transformer Networks}



In this section, we apply MK to language tasks, focusing on fine-tuning downstream tasks using MK-enhanced pretrained T5 models. We expand a pretrained model's kernels into majority kernels by replicating them multiple times, forming our initial MK pretrained model. Using this model, we fine-tune on GLUE tasks with our MK optimization algorithm and present the results for various algorithms:
\begin{itemize}
    \item \textbf{\textsc{Baseline}}. Fine-tuning a pretrained T5 language model. This model is the T5 model with the configuration ``Small'', and is pretrained as described in \cite{t5_paper}.
    \item \textbf{\textsc{Majority Kernels}}. 
    Starting with a small T5 pretrained model (refer to \cite{t5_paper}), we convert its kernels into MK kernels with expansion factor~$e=3$. We then fine-tune it using our MK algorithm.
\end{itemize}
Table~\ref{tab:t5_algorithms} summarizes the results for GLUE language tasks. 
A full breakdown for the various tasks and some extra details can be found in Appendix~\ref{apdx:t5_experiment}. To make sure the performance is not tied to one pretrained checkpoint, for every run performed below we used a new pretrained checkpoint. Our training cost increase for T5 was negligable. Finally, the appendix holds comparison with Model Soups \cite{wortsman2022model} - a baseline the performs ensembling in the parameter space.

\begin{table}[h]
    \centering
    \begin{small}
    \begin{tabular}{c|c}
    \toprule
         \textbf{Algorithm} & \textbf{Glue average}  \\ \midrule
         \textsc{Baseline} & 80.3 $\pm$ 0.12 \\
         \textsc{Majority Kernels} & \textbf{80.9} $\pm$ 0.3 \\
    \bottomrule
    \end{tabular}
    \end{small}
    \vspace{5pt}
    \caption{Results on GLUE language tasks for our algorithm compared to vanilla training.}
    \label{tab:t5_algorithms}\vspace{-10pt}
\end{table}


Appendix~\ref{apdx:t5_experiment} contains figures illustrating the evaluation curves for various tasks during \textsc{Baseline} and \textsc{Majority Kernels} training. A key observation is that, unlike previous experiments where our algorithm required extra steps, in this case, it achieves higher performance more rapidly.
Regularization limits the model's complexity, discouraging it from fitting too closely to the training data of the primary task, which can include noise or irrelevant patterns. This ensures the model captures more general patterns, making it less tailored to the nuances of the primary task but more adaptable and easily fine-tuned for new, related tasks, as it has not over-committed to the specifics of the original training data.







\section{Conclusions}

The MK algorithm is a robust advancement in machine learning as it achieves computational efficiency during training, automates the distillation of inference parameters, and improves parameter weight optimization, which in turn smooths the training loss landscape to avoid local minima. Additionally, the MK algorithm's implicit regularization analysis reveals a beneficial second-order dependency in the modified loss, introducing stochastic regularization for smoother optimization and a bias towards flatter minima. This sharpness analysis indicates the algorithm averages out training randomness, leading to a stable and comprehensive representation of the model, particularly advantageous when capturing diverse data aspects.




\newpage


\bibliographystyle{unsrtnat}

\bibliography{references}

\appendix
\newpage

\section{Proof of Theorem~\ref{th:bea}}
\label{sec:proof-main-app}
\begin{restatable}[Backward Error Analysis]{theorem}{bea}
Let $L$ be a sufficiently differentiable function on the parameter space $\theta \in \mathbb{R}^n$. The modified loss when using Majority Kernels is,
\begin{equation*}
\tilde{L}_{MK}(\thetam) = L(\thetam) + \frac{\ell}{4} \lVert \nabla L(\thetam) + \nabla^2 L(\thetam)\cdot\epsilon \rVert^2 + \nabla L(\thetam) \cdot \epsilon.
\end{equation*}
%
%
where $\epsilon \in \mathbb{R}^n$ is 
is the random perturbation of the virtual parameters.
\end{restatable}
\begin{proof} Our proof follows the main theorem in \cite{barrett2020implicit} (Theorem 3.1).
We want a modified equation with correction terms of the form:
\begin{equation}
\label{eq:modified_f}
    \tilde{f}(\thetamp) = f(\thetamp) + \ell C_1(\thetamp).
\end{equation}

The Taylor series expansion of the true solution $\theta(t+h)$ is,
\begin{equation}
    \theta(t+h) = \theta(t) + \ell f(\theta) + \frac{\ell^2}{2} f'(\theta) f(\theta) + {\cal O}(\ell^3).
\end{equation}

Now, replacing Eq.~\ref{eq:modified_f} into $\theta(t+h)$ we get, 
\begin{samepage}
\begin{align}
\label{eq:real_theta}
    \theta_m&(t+h) = \theta(t) + \ell(f(\thetamp) + \ell C_1(\thetamp)) + \frac{\ell^2}{2}f'(\thetamp) f(\thetamp) \\\nonumber
    %
    %
    &=\theta(t) + \ell f(\thetamp) + \ell^2\left(C_1(\thetamp) + \frac{1}{2}f'(\thetamp) f(\thetamp) \right).
\end{align}
\end{samepage}

The numerical method with a first order Euler method using $f(\theta)$ for consistency is,
\begin{equation}
\label{eq:numerical_theta}
  \theta_{n+1} = \theta_n + \ell f(\theta_n).    
\end{equation}

To get $\theta_m(n\ell)=\theta_n$ for all $n$, we must have $\theta_m(t+\ell)$ matching the numerical method. Comparing like powers of $\ell$ in equations \ref{eq:real_theta} and \ref{eq:numerical_theta} yields recurrent relations for the correction functions:
\begin{align*}
    f(\theta_n) =& f(\thetamp), \ \ 
    C_1(\thetamp) + \frac{1}{2}f'(\thetamp) f(\thetamp) = 0.
\end{align*}

So the correction term $C_1$ becomes:
$
C_1(\thetamp) = -\frac{1}{2}f'(\thetamp) f(\thetamp).
$

For our algorithm, $f(\thetamp)=f(\thetam + \epsilon)$. Doing the first order Taylor expansion of this term yields,
$$
f(\thetam + \epsilon)=f(\thetam) + f'(\thetam) \cdot \epsilon,
$$
where since $\epsilon$ is a vector, the product is matrix-vector product. The correction terms become:
$$
f(\theta_n) = f(\thetam) + f'(\thetam) \cdot \epsilon
$$
and
\begin{align*}
C_1(\thetamp)=&-\frac{1}{2}f'(\thetam+\epsilon)f(\thetam+\epsilon) =-\frac{1}{2}(f'(\thetam) +  f''(\thetam)\cdot \epsilon)(f(\thetam) + f'(\thetam)\cdot \epsilon ) = \\
=&-\frac{1}{2}\left(f'(\thetam)f(\thetam) + f''(\thetam)f(\thetam)\cdot\epsilon +  f'(\thetam)^2\cdot\epsilon  \right)
+\mathcal{O}(\epsilon) 
\end{align*}

Finally, the modified vector field in Eq.~\ref{eq:modified_f} becomes:
\begin{align*}
    \tilde{f}(\thetamp) &= f(\thetam) + f'(\thetam) \cdot \epsilon 
    -\frac{\ell}{2}\left(f'(\thetam)f(\thetam) + f''(\thetam)f(\thetam)\cdot\epsilon +  f'(\thetam)^2\cdot\epsilon \right)
    +{\cal O}(\epsilon^2).
\end{align*}

Substituting the vector field definition, the first part is:
$
    f(\thetam) + f'(\thetam) \cdot \epsilon = -\nabla L(\thetam) - \nabla^2 L(\thetam) \cdot \epsilon,
$\\
and the second part is:
\begin{align*}
    f'(\thetam)f(\thetam) + f''(\thetam)f(\thetam)\cdot\epsilon + f'(\thetam)^2\cdot\epsilon &=\nabla (\nabla L(\thetam))\nabla L(\thetam) + (\nabla^3 L(\thetam))\nabla L(\thetam)\cdot\epsilon + (\nabla^2 L(\thetam))^2\cdot\epsilon\\
    & = \frac{1}{2}\nabla \lVert \nabla L(\thetam) + \nabla^2 L(\thetam) \cdot\epsilon \rVert^2 + \mathcal{O}(\epsilon^2)
\end{align*}

Removing the negligible $\mathcal{O}(\epsilon^2)$ term, the modified loss for learning rate $\ell$ is
\begin{align*}
    \tilde{L}_{MK}(\thetam) = L(\thetam) + \frac{\ell}{4} \lVert \nabla L(\thetam) + \nabla^2 L(\thetam)\cdot\epsilon \rVert^2 + \nabla L(\thetam) \cdot \epsilon.
\end{align*}
which concludes the proof.\end{proof}

\begin{corollary}
Our approach implicitly incorporates 
two new elements, the Hessian-based term \( - \nabla^2 L(\thetam) \cdot \epsilon \), which introduces a second-order characteristic to the optimization process, and a distortion-based gradient term.
Unlike traditional second-order methods that utilize an inverted Hessian to determine the direction of steepest descent, the direct application does not aim to pinpoint the exact descent direction; instead, it modulates the gradient update to reflect the underlying curvature of the loss surface. Stochastically adding the Hessian term in the penalizing norm should bias to a solution with not only small gradient norm, but also small Hessian norm. From the point of view of the modified loss, the MK algorithm introduces extra stochastic regularization that will offer a smoother navigation of the optimization landscape and bias toward flatter minima.

\end{corollary}

\newpage

\section{Stochastic sharpness aware minimization}\label{section:theory}

Conventional Sharpness-Aware Minimization (SAM) \citep{foret2020sharpness} aims to find parameters that not only minimize the training loss, $L(\theta)$, but also maintain a low loss in the vicinity of $\theta$, thereby leading to solutions that generalize better. SAM achieves this by considering both $L(\theta)$ and $L(\theta + \epsilon)$, where $\epsilon$ is a perturbation that maximizes the loss within a defined neighborhood of $\theta$.

In our scenario, the perturbation happens naturally via our stochastic approach. The implicit perturbation \( \epsilon \) reflects the variability during training and is not necessarily the worst-case perturbation. Therefore, we can consider the expected value of the loss due to the perturbation \( \epsilon \) within a defined neighborhood. The bound on the generalization loss would be more about the expected sharpness rather than the maximum sharpness.

\begin{restatable}[PAC-Bayesian Bound with Stochastic Weights]{lemma}{bayesian}
\label{lemma:bayesian}

A network is parameterized by an extended set of weights $\we^{(l)}$ per layer, $l\in[L]$, and the parameters $\theta$ the network operates in are the result of the stochastic aggregation defined in Eq.~\ref{eq:w}. Let $L_D(\theta)$ and $L_S(\theta)$ denote the true and empirical loss functions, respectively. Let $Q$ be the distribution of model parameters induced by the stochasticity in the weights, and let $\Omega_u(e)$ be the uniform distribution over the extended weight space.
For any data distribution $D$, number of samples $m$, training set $S \sim D$, and prior distribution $P$ on parameters $\thetae$, posterior distribution $Q$, for any $0 < \delta$, with probability $1-\delta$ over the draw of training data,
then the expected true loss under $Q$ can be bounded as follows~\citep{chatterji19generalization}:
    \begin{equation*}
        E_{\theta \sim Q}[L_D(\theta)] \leq E_{\theta \sim Q}[L_S(\theta)] + \sqrt{\frac{\text{KL}(Q\|P) + \log \frac{m}{\delta}}{2(m-1)}},
    \end{equation*}
    where $\text{KL}(Q\|P)$ is the Kullback-Leibler divergence between the distribution $Q$ and a prior distribution $P$.
    
    The expected empirical loss under $Q$ can be expressed as:
    \begin{align*}
    E_{\theta \sim Q}[L_S(\theta)] &= E_{p\sim \Omega_u(e)}[L_S(\thetamp)] \leq L_S(\thetam) + \Delta
    \end{align*}
    where
    \begin{align*}
      \Delta &= \left| E_{p\sim \Omega_u(e)}[L_S(\thetamp)] - L_S(\thetam) \right|,
    \end{align*}
    is the stochastic sharpness term that represents the deviation of the expected empirical loss under the random weights from the empirical loss of the mean weights. It captures the sensitivity of the empirical loss to fluctuations in the model parameters.

\end{restatable}






Finally, the stochastic approach affects the convergence behavior of the algorithm, potentially leading to more stable but slower convergence (see Lemma~\ref{lemma:uniform_fallback} which also showcase the importance of the random probabilities).
\begin{restatable}[Reduced Learning Rate with uniform probabilities]{lemma}{UniformFallback}
\label{lemma:uniform_fallback}

Let \(\ell\) be the standard learning rate in a conventional gradient descent algorithm. If \( e > 0 \) represents the extension factor and $p=\left( \frac{1}{e}, \ldots, \frac{1}{e} \right)$ the effective learning rate w.r.t the virtual layer parameters $w$ is \( \frac{\ell}{e} \), that is:
$$
w \gets w - \frac{\ell}{e} (\nabla_w L(w)).
$$

\end{restatable}



\begin{proof}

Our algorithm optimizes for \(\we\) instead of the conventional weight parameters \(w\).

When performing gradient descent with learning rate $\ell$ parameters are updated as follows,
$$
\we \gets \we - \ell \nabla_{\we} L(\we).
$$

Using the chain rule,
$$
\nabla_{\we} L(\we) = \nabla_w L(w) \cdot \frac{\partial w}{\partial \we},
$$
where \(\nabla_w L(w)\) represents the gradient of the loss function \(L\) with respect to \(w\), and \(\frac{\partial w}{\partial \we}\) is the partial derivative of \(w\) with respect to \(\we\), which in the given scenario is represented as a scaling matrix where each element is scaled by \(\frac{1}{e}\). This scaling matrix is applied element-wise to \(\nabla_w L(w)\) to obtain the gradient with respect to \(\we\).

Thus, when performing gradient descent with our algorithm with learning rate $\ell$ and uniform $p$, the implicit update of the extended kernel $\we$ with respect to the conventional parameters $w$ looks as follows,
$$
\we \gets \we - \frac{\ell}{e} (\nabla_w L(w) \otimes \mathbf{1}_e),
$$
where $\otimes$ is the Kronecker product and \( \mathbf{1}_e \) is the \( 1 \times e \) vector of ones.

Under the condition where the parameters \(p\) follow a uniform distribution as described in Eq.~\ref{eq:w}, the effective forward propagation step is implicitly using,
$$
w = \frac{1}{e} \sum_{k\in[e]} \we_{k},
$$
and since moving every element in a vector by a constant $c\in\mathbb{R}$ moves the average of that vector by $c$, we get that
$$
w \gets w - \frac{\ell}{e} (\nabla_w L(w)).
$$
\end{proof}

\newpage

\section{Related Work}
\label{sec:related}


Standard approaches for obtaining a model that is amenable to inference time constraints rely on a post-training processing stage via various methods. One class of popular methods concern model compression and quantization. A popular approach to model quantization is to truncate the model weights to limited bits of precision such as 4-bit quantization or 8-bit quantization \citep{banner2019post}. Typically quantizing the learned model weights leads to a loss in performance and one often needs an additional round of fine-tuning on the quantized weights \cite{fan2020training, bai2018proxquant, nagel2020up}. There have also been efforts to perform post training quantization without the need for additional finetuning \citep{banner2019post, cai2020zeroq}. Other approaches include hardware aware quantization \citep{wang2019haq}, quantization based on $k$-means \citep{gong2014compressing} and approaches exploring extreme one-bit quantization \citep{bai2020binarybert}. In a similar vein, approaches based on the lottery ticket hypothesis \citep{frankle2018lottery} aim to prune connections within a pretrained network which amounts to zeroing out entries of the learned weight matrices. See the survey of \citet{gholami2022survey} for an in-depth discussion of quantization.

An alternative to model compression is the idea of knowledge distillation \citep{buci2006model, hinton2015distilling}. Given a large pretrained teacher network, distillation involves training a smaller student network, typically of the same architecture as the larger one, to mimic the behavior of the larger network. Hence the larger model acts as a source for labeled supervision and it is often the standard practice to train the smaller model over the smoothed labels (the full logit distribution of the larger model). There have also been recent works exploring the idea of online distillation \citep{harutyunyan2023supervision} or co-distillation \citep{anil2018large} where the teacher and the student models are trained simultaneously.

Our proposed majority kernels have similarities to the classical notion of model ensembling. There is a rich body of work on principled techniques such as bagging \citep{breiman1996bagging} and boosting \citep{freund1997decision} for producing an ensemble of smaller base models. In recent years it has been observed that empirically, even a simple averaging of independently trained networks produces strong ensembles \citep{lakshminarayanan2017simple}. There have also been efforts to produce an ensemble of multiple models via a single round of training \citep{huang2017snapshot}. While an ensemble model leads to performance benefits, applying it in inference constrained settings still requires compression techniques such as knowledge distillation. Our proposed approach can be viewed as a way to avoid that by implicitly performing model ensembling in the  parameter space itself. A similar intuition underlies the standard practice of dropout regularization \citep{srivastava2014dropout}, but dropout does not produce a smaller model at the end of training. Our approach is complimentary to dropout, and can in fact be used in conjunction with it.

Finally, there have been recent approaches towards maintaining inference efficiency while simultaneously leveraging the capabilities of a larger model during training time. The sparse mixture-of-experts (MoE) architecture \citep{shazeer2017outrageously} aims to train a large model consisting of small experts and each example is routed to only a few of the experts. Another approach involves adapting a large pretrained network for many downstream tasks via adding low rank updates to the weight matrices \citep{hu2021lora}. Finally, the recent work of \citet{kudugunta2023matformer} aims to produce multiple models of various sizes as a result of a single training run. This is achieved by training over a loss averaged over the loss of the constituent models.

\newpage

\section{Adversarial Probabilities}\label{sec:adv-majority}
In this appendix, we will discuss a variation of the Majority Kernels algorithm presented in the paper, where we apply adversarial perturbations to the probabilities. We hope to shed light on some of the design choices, namely, the empirical reason behind sticking with simple stochastic random probability choices, which provide a simple yet effective training method.

The idea behind the adversarial probabilities is to optimize for the following loss in training,
$$
\min_{\we}\ \max_p L(\langle \we \rangle_p).
$$
We do so by making the probabilities learnable, setting them all to $1/e$ where $e$ is the expansion factor.
The full algorithm can be found below (Algorithm \ref{adv_only}),

\begin{algorithm}[h]
\caption{Adversarial Only Majority kernels.}
\begin{algorithmic}
\STATE For every layer $l$:  Initialize $\we\in\mathbb{R}^{n\times m \times e}$, $b\in\mathbb{R}^{m}$
\STATE initialize: $\ell \gets$ learning rate, $\epsilon \gets $ small positive
\WHILE{$s < \mathrm{max\_steps}$}
  \STATE $B \gets $\texttt{NewBatch}()
  \STATE $p \gets 1/e^{m\times e}$\\
  \STATE gradient$_p \gets$ \texttt{ComputeGradient}($\langle \we \rangle_p$,$b$,$B$)\\
  \STATE $p \gets p + \epsilon\, $gradient$_p$ \\
  \STATE $p \gets $\texttt{Normalize}($p$)\\
  \STATE gradient$_{\we,b} \gets$\texttt{ComputeGradient}($\langle \we \rangle_p$,$b$,$B$)\\
  \STATE $\we$, $b \gets \we$, $b - \ell\, $gradient$_{\we,b}$ \\
  \STATE $s+=1$
\ENDWHILE
\STATE return $\langle \we \rangle$, $b$ \algorithmiccomment{Return trained parameters for inference}
\end{algorithmic}\label{adv_only}
\end{algorithm}

Empirical results showed that while performing great at the beginning of training, this algorithm led to over-fitting later on. This was the case also when trying to learn the probabilities as part of the model, i.e., train with the following loss,
$$
\min_{\we,p} L(\langle \we \rangle_p).
$$
We believe that the over-fitting happens from the kernels becoming equal, which leads to equal adversarial probabilities (uniform), which is equivalent to training with low learning rate (See Lemma~\ref{lemma:uniform_fallback}). To address this, we introduced a random element to the $p$ to prevent the kernels from becoming equal by adding randomness if the adversarial probabilities are equal. This led to the algorithm reported in Table~\ref{tab:cnn_algorithm}. The algorithm explain below.

\begin{algorithm}[h]
\caption{Adv-Majority kernels.}
\begin{algorithmic}
\STATE For every layer $l$:  Initialize $\we\in\mathbb{R}^{n\times m \times e}$, $b\in\mathbb{R}^{m}$
\STATE initialize: $\ell \gets$ learning rate, $\epsilon \gets $ small positive
\WHILE{$s < \mathrm{max\_steps}$}
  \STATE $B \gets $\texttt{NewBatch}()
  \STATE $p \gets 1/e^{m\times e}$\\
  \STATE gradient$_p \gets$ \texttt{ComputeGradient}($\langle \we \rangle_p$,$b$,$B$)\\
  \STATE $p \gets p + \epsilon\, $gradient$_p$ \\
  \STATE $p \gets $\texttt{Normalize}($p$)\\
  \STATE $u \gets $KL\_Divergence($p$, $1/e^{m\times e}$) / $\log e$
  \STATE $random_p \gets \texttt{NormalizedExponentialRnd}()$
  \STATE $p \gets u \cdot p  + (1-u) \cdot random_p$
  \STATE gradient$_{\we,b} \gets$\texttt{ComputeGradient}($\langle \we \rangle_p$,$b$,$B$)\\
  \STATE $\we$, $b \gets \we$, $b - \ell\, $gradient$_{\we,b}$ \\
  \STATE $s+=1$
\ENDWHILE
\STATE return $\langle \we \rangle$, $b$ \algorithmiccomment{Return trained parameters for inference}
\end{algorithmic}\label{adv_mj}
\end{algorithm}

Figure~\ref{fig:adv_analysis} shows the difference between the learning curves with and without randomness (i.e., Algorithm~\ref{adv_only} vs Algorithm~\ref{adv_mj}). It is easy to see that the randomness prevents the over-fitting and may lead a slightly better algorithm than the Majority Kernels; However, we should mention that this algorithm has high overhead compared to regular training, as it is calculating the gradient twice, and thus, this algorithm is not intended as a primary contribution.

\begin{figure}[h]
    \centering
    \includegraphics[scale=0.25]{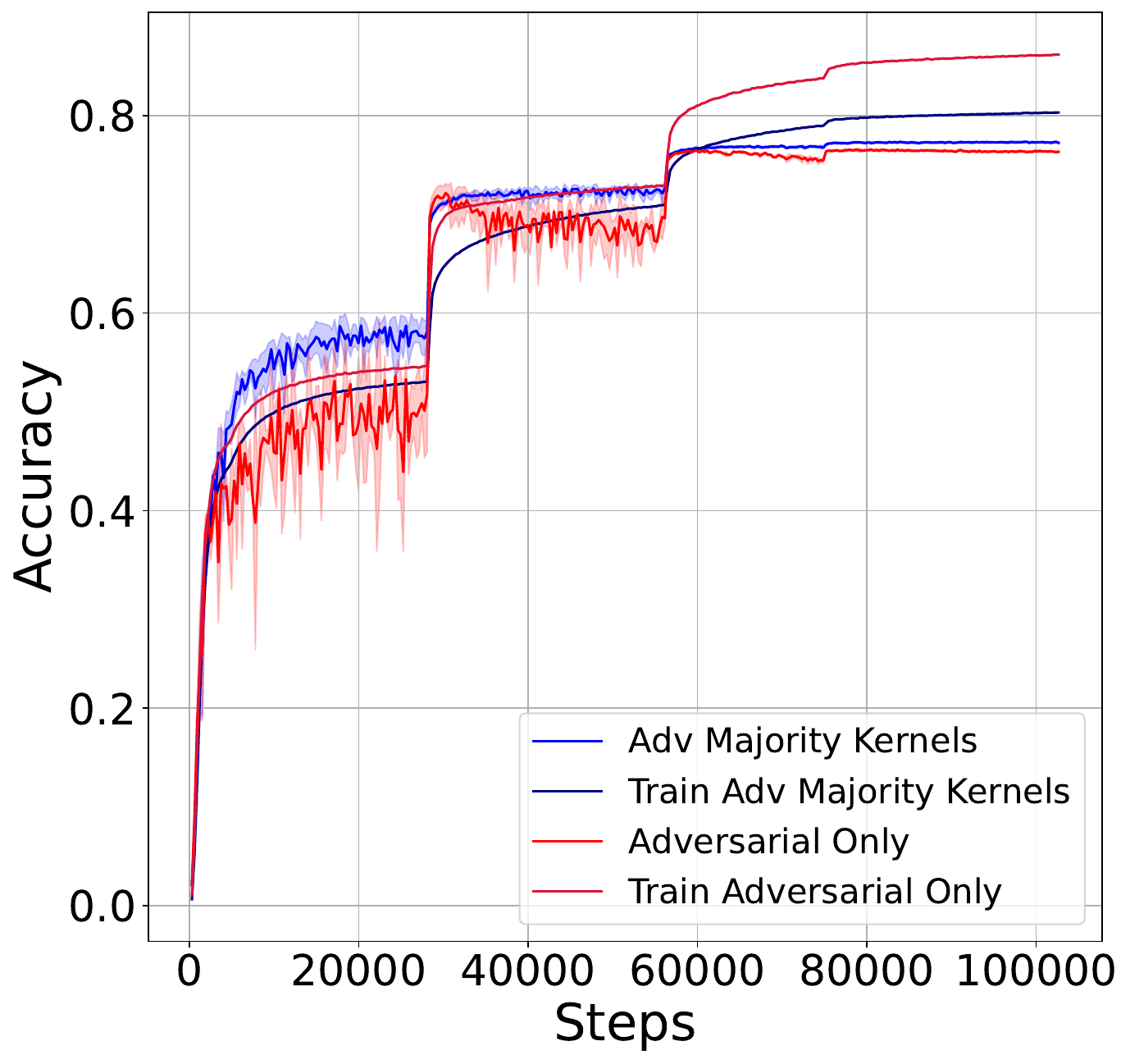}
    \caption{Train and eval curves for algorithms~\ref{adv_only} and~\ref{adv_mj}}
    \label{fig:adv_analysis}
\end{figure}

\section{Glue Experiment Breakdown}\label{apdx:t5_experiment}
In this appendix, we provide details of the Glue experiment, starting with Table~\ref{tbl:full_glue}, which breaks down the performance of various tasks against vanilla training.

\begin{table*}[htb]
\hspace{-30pt}
\begin{tiny}
\begin{tabularx}{1.15\textwidth}{cXXXXXXXXXXXXX}
\toprule
Model & Glue avg & COLA Matthew's & SST acc & MRPC f1 & MRPC acc & STS-b pearson & STS-b \mbox{spearman} & qqp acc & qqp f1 & MNLI-m & MNLI-mm & QNLI & RTE\\
\midrule
\textsc{Baseline} & 80.3$\pm$0.1 & 36.88$\pm$1 & 92.43$\pm$0.2 & 90.85$\pm$0.2 & 87.58$\pm$0.3 & 88.17$\pm$0.3 & 88.03$\pm$0.3 & 88.02$\pm$0.1 & 91.16$\pm$0.1 & 83.96$\pm$0.1 & 83.34$\pm$0.1 & \textbf{90.13}$\pm$0.2 & 72.44$\pm$0.9 \\
\textsc{Majority Kernels} & \textbf{80.9}$\pm$0.3 & \textbf{39.69}$\pm$1.7 & \textbf{92.77}$\pm$0.2 & \textbf{91.39}$\pm$0.2 & \textbf{88.23}$\pm$0.2 & \textbf{88.73}$\pm$0.4 & \textbf{88.71}$\pm$0.1 & 88.03$\pm$0.1 & 91.17$\pm$0.1 & 83.9$\pm$0.2 & \textbf{83.66}$\pm$0.3 & 89.75$\pm$0.5 & 73.04$\pm$0.9 \\ \bottomrule
\end{tabularx}
\end{tiny}
\caption{Performance of the various models on downstream tasks (Glue tasks).}
\label{tbl:full_glue}
\end{table*}

In addition, Figures~\ref{fig:glue_cola},~\ref{fig:glue_mrpc},~\ref{fig:glue_rte},~\ref{fig:glue_sst2} and \ref{fig:glue_stsb} show eval curves on various tasks revealing that our algorithm achieves peak faster than vanilla training. In addition, our algorithm is less prune to overfitting when over trained.

\begin{figure}[h]
    \centering
    \includegraphics[scale=0.25]{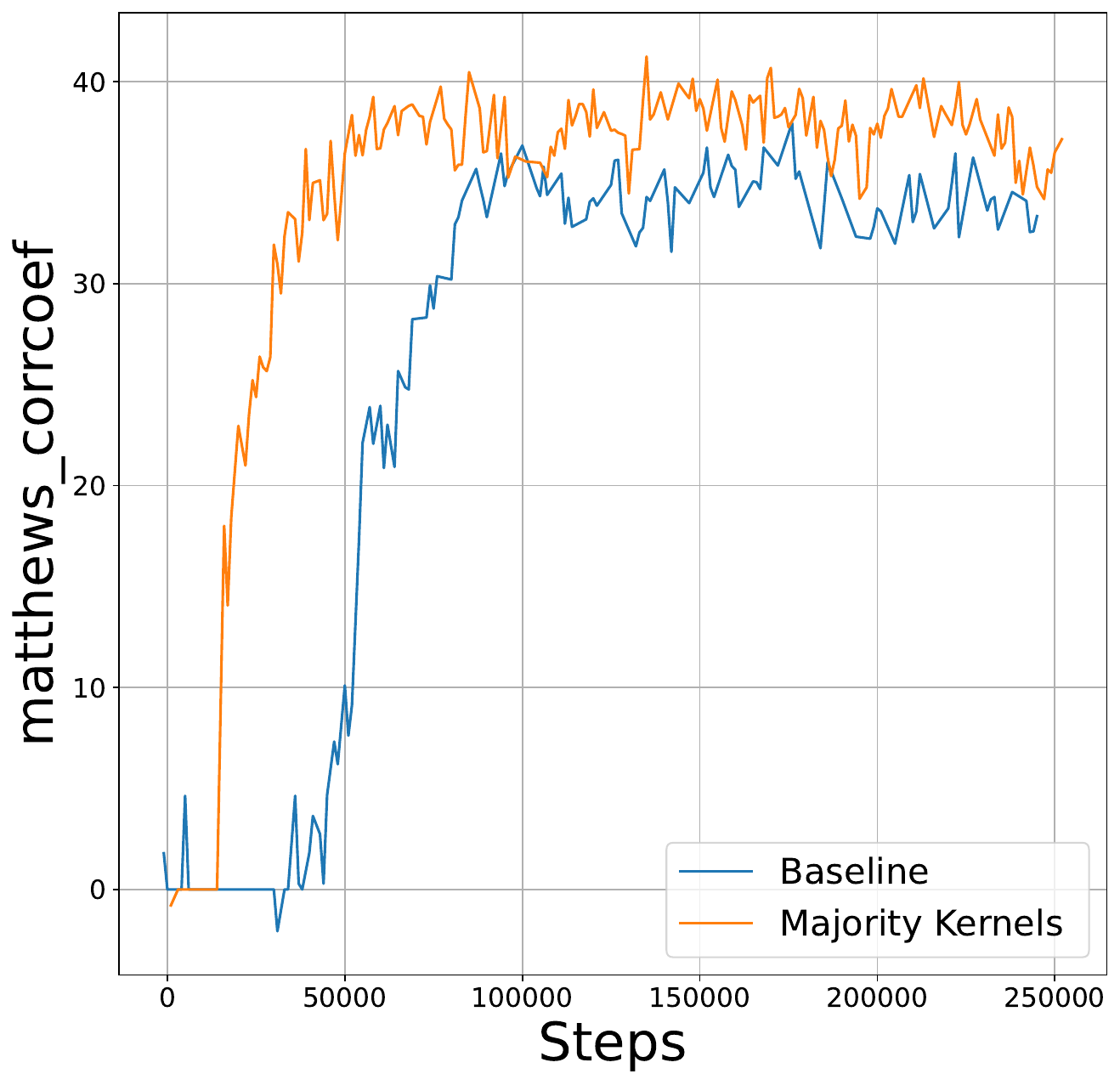}
    \caption{Eval curves for our algorithm compared to vanilla training on Glue Cola}
    \label{fig:glue_cola}
\end{figure}

\begin{figure}[h]
    \centering
    \includegraphics[scale=0.25]{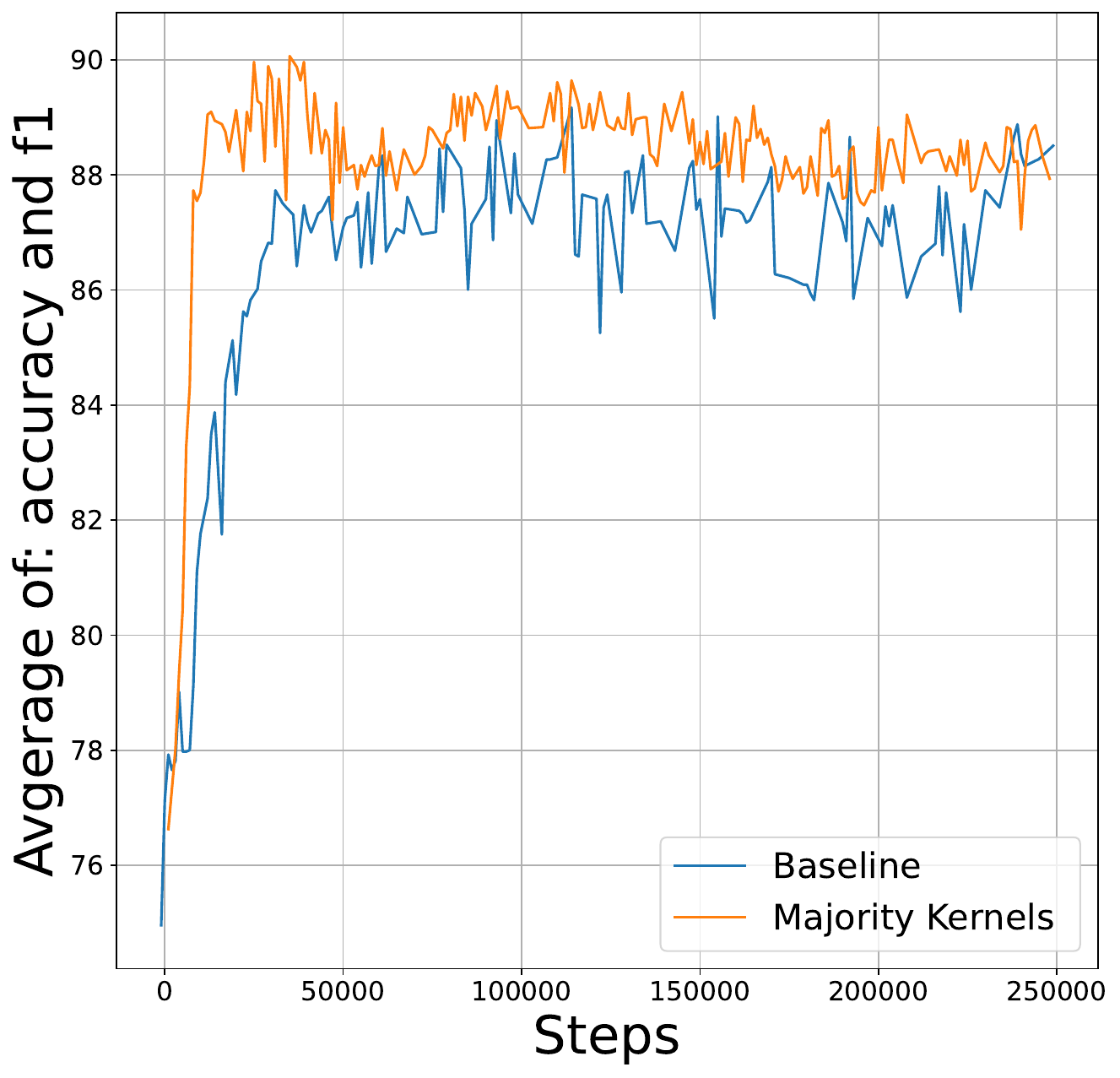}
    \caption{Eval curves for our algorithm compared to vanilla training on Glue MRPC}
    \label{fig:glue_mrpc}
\end{figure}

\begin{figure}[h]
    \centering
    \includegraphics[scale=0.25]{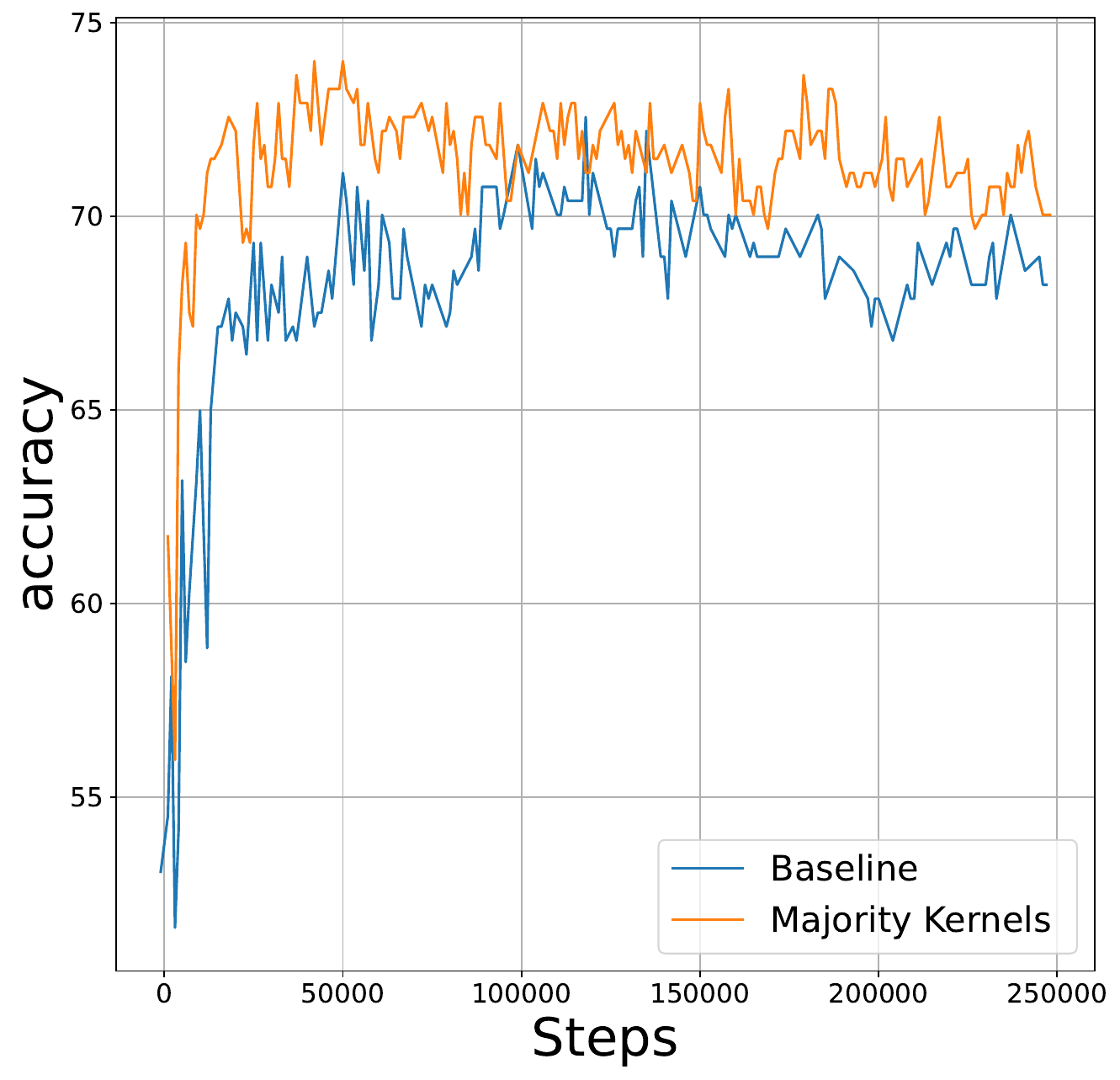}
    \caption{Eval curves for our algorithm compared to vanilla training on Glue RTE}
    \label{fig:glue_rte}
\end{figure}

\begin{figure}[h]
    \centering
    \includegraphics[scale=0.25]{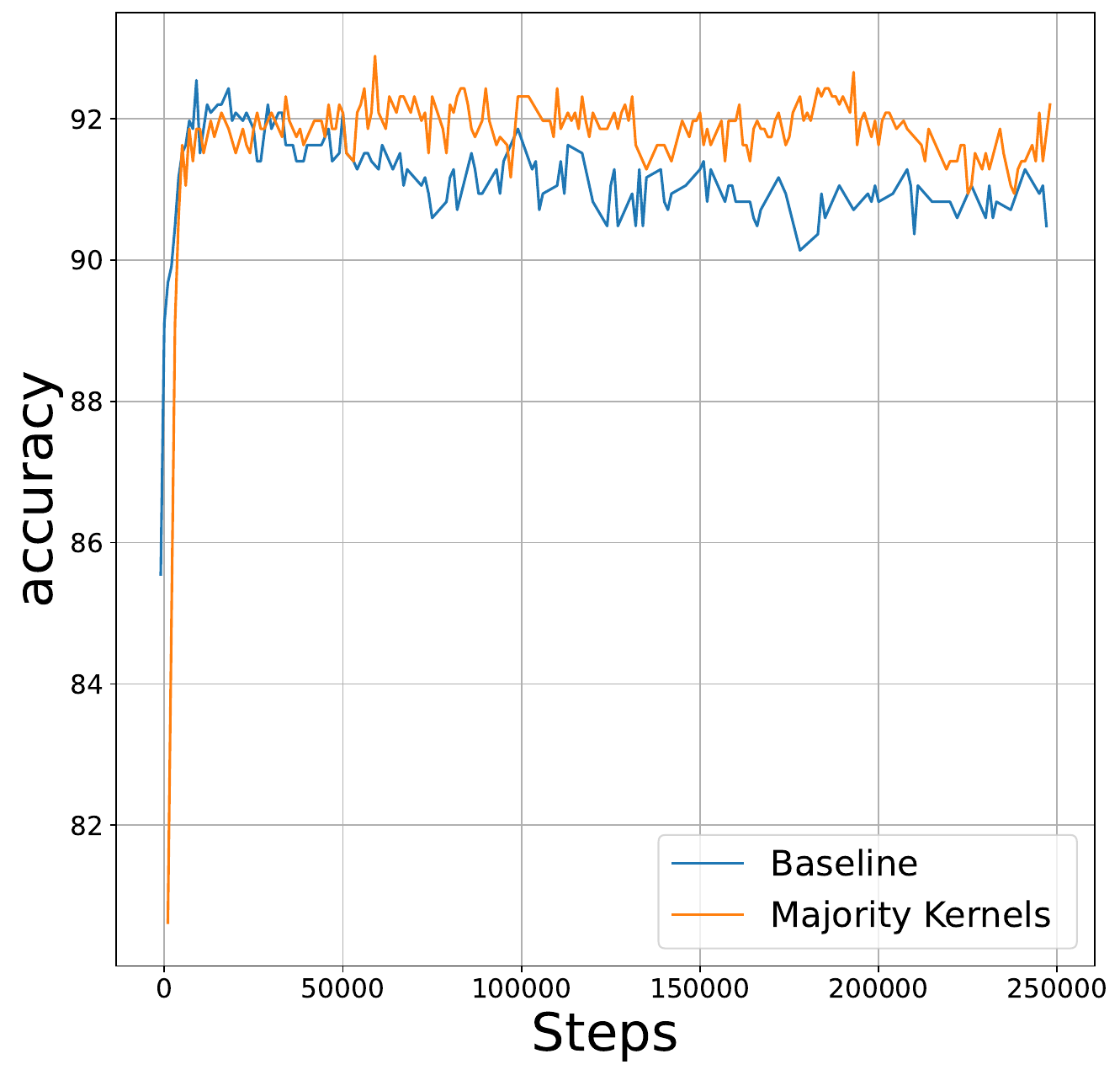}
    \caption{Eval curves for our algorithm compared to vanilla training on Glue SST2}
    \label{fig:glue_sst2}
\end{figure}

\begin{figure}[h]
    \centering
    \includegraphics[scale=0.25]{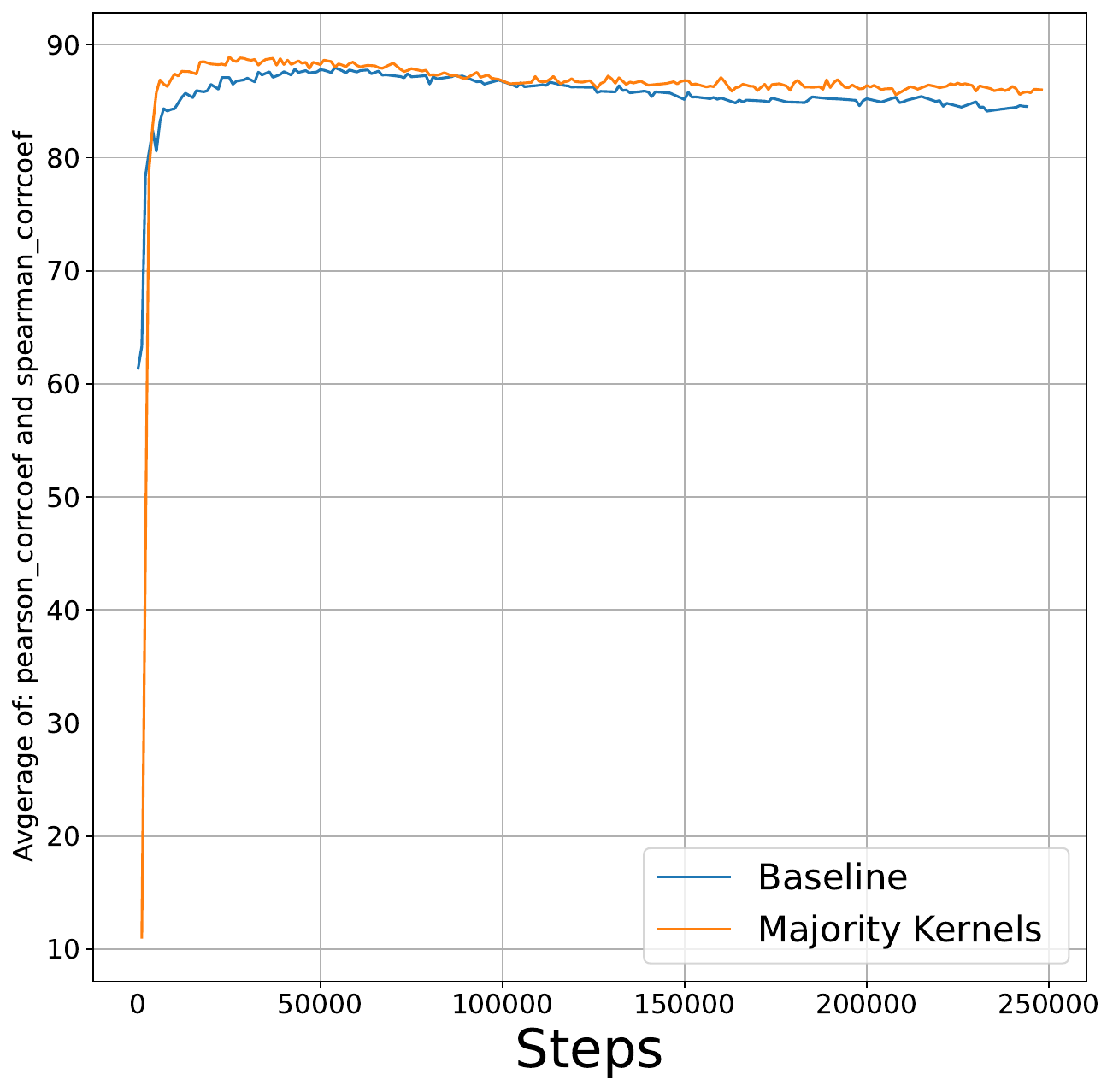}
    \caption{Eval curves for our algorithm compared to vanilla training on Glue STSb}
    \label{fig:glue_stsb}
\end{figure}

Finally, in \cite{wortsman2022model}, the authors discover that no alignment is needed between kernels when finetuning multiple times provided that we finetune on the same pretrained model. The authors evaluate performance on T5, with size configuration that matches ours but with higher expansion scale (more kernels to average). Table~\ref{tab:t5_soup_compare} below shows performance comparison on GLUE tasks in which the authors publish performance.

\begin{table}[h]
    \centering
    \begin{scriptsize}
    \begin{tabular}{c|c|c|c|c}
    \toprule
         \textbf{Algorithm} & \textbf{MRPC} & \textbf{RTE} & \textbf{CoLA} & \textbf{SST-2}  \\ \midrule
         \textsc{Baseline}  & 89.216 &  72.443  & 36.883 & 92.433 \\
         \textsc{Majority Kernels} & \textbf{89.816} & \textbf{73.043} & 39.696 & \textbf{92.776}\\
         \textsc{Model Soups - uniform} & 82.7 & 61.7 & 10.4 & 91.1\\
         \textsc{Model Soups - greedy}  & 89.7 & 70 & \textbf{43} & 91.7\\
    \bottomrule
    \end{tabular}
    \end{scriptsize}
    \vspace{0.5cm}
    \caption{Results on GLUE language tasks for our algorithm compared Model Soups (Table J.1 in \cite{wortsman2022model}).}
    \label{tab:t5_soup_compare}
\end{table}

Our algorithm have greater boost on most tasks with only averaging three kernels. While the algorithm mentioned above share similarity with ours in that it averages multiple kernels to create one inference one, one substantial difference is that our algorithm does not require the compute of finetuning multiple times, nor the additional engineering complexity of finding the right step where performance peaks for early stopping. We require one run with a slight overhead per step, and our evaluation is continuous.

\section{The Subset Selection Baseline}
\label{sec:subset}

We can view the shrinking of model dimension overparameterization as a combinatorial subset selection algorithm, and one of the predominantly used subset selection method is based on submodular maximization \citep{Nemhauser1978Submod,fujishige2005submodular}. To describe our method we first focus on a 1-layer network, i.e., $f(x) = v^T \sigma(W \cdot x)$ where $x \in \mathbb{R}^d, W \in \mathbb{R}^{m \times d}$ and $v \in \mathbb{R}^m$. For a given overparameterization factor $e > 1$, we initialize the network with parameters $\{v_0, U_0\}$ where $v_0 \in \mathbb{R}^m$ and $U_0 \in \mathbb{R}^{e\cdot m \times d}$. At time $t$, before each step of gradient update, i.e., a forward and backward pass, we first invoke a combinatorial subset selection procedure to select the best $m$ rows (neurons) out of  the $e \cdot m$ rows in $U_t$. At the end of training we again invoke the subset selection procedure to select the best $m$ rows to output the final network. The above approach can be easily extended to deeper networks by independently invoking the subset selection procedure for each hidden layer in the network.

We next describe the subset selection procedure. The core idea stems from the fact that we should aim to select neurons that have the most utility, i.e., achieve low loss overall and at the same time aim to avoid selecting redundant neurons, i.e., keep the selected network small. Hence we need to balance notions of utility and diversity, a setting tailor made for submodular optimization. Given $U \in \mathbb{R}^{e\cdot m \times d}$, for each row $i \in [e \cdot m]$ let $u(i)$ denote it's perceived utility. Furthermore let $s(i,j)$ denote the cosine similarity between rows $i$ and~$j$. Then for a given subset $S$ of the rows we consider the following pairwise submodular objective that evaluates the effectiveness of $S$
\begin{align}
    f(S) = \alpha \sum_{i \in S} u(i) - \beta \sum_{(i,j) \in S} s(i,j)
\end{align}
where $\alpha, \beta$ are hyperparameters. By appropriate choices of the parameters and the similarity functions, it can be shown that the above objective is both submodular and monotonically non-decreasing. Note that evaluating $f(S)$ involves $O(|S|^2)$ computation due to the presence of pairwise terms. This can be computationally prohibitive for layers that have thousands of neurons. Hence as a practical approximation we first consider a $t$-nn graph $G = (V,E)$ over the $k \cdot m$ rows where each row is only connected to its $t$ nearest neighbors. In our experiments we pick $t$ to be a small value ($t=10$). Furthermore, we use the norm of row $i$ as a proxy for the utility of neuron $i$. Let $c$ be the constant term defined as $c = \max_{\ell} \sum_{j: (\ell,j) \in E} s(\ell, j)$. Then we define the utility as $u(i) = \|U_i\| + c$. Hence our final objective is as follows
\begin{align}
\label{eq:submod}
    S^* &= \argmin_{S: |S|=m} f(S)\\
    f(S) &= \alpha \sum_{i \in S} u(i) - \beta \sum_{\substack{(i,j) \in S\\ (i,j) \in E}} s(i,j).
\end{align}

It is easy to see that the above is a monotone submodular objective for which a simple greedy algorithm achieves a $1 - \frac{1}{e}$-approximation \citep{Nemhauser1978Submod}. The full training procedure based on the above approach is described in \cref{subset_algo}.

\begin{algorithm}
\caption{The subset selection algorithm for training.}
\begin{algorithmic}
\STATE For every layer $r$:  Initialize $W_r\in\mathbb{R}^{n_r\times m_r \times e}$, $b_r\in\mathbb{R}^{m_r \times e}$.
\WHILE{$step < \mathrm{max\_steps}$}
  \STATE $B \gets $New Batch
  \STATE $\forall r$, compute the subset $S_r$ of $m_r$ rows via the greedy algorithm for the objective in \cref{eq:submod}.
  \STATE Train one step with parameters $W_r[S_r]$, $b_r[S_r]$ and batch $B$.
  \STATE $\forall r$: update $W_r$, $b_r$
  \STATE $step~=~step + 1$
\ENDWHILE
\STATE {\bf return} $W^m$ by again invoking the greedy algorithm for submodular optimization for each layer.
\end{algorithmic}\label{subset_algo}
\end{algorithm}

\newpage

\section{Experiments on Fully Connected Networks}\label{sec:experiments-dnn-app}
We consider training of vanilla feedforward networks on the CIFAR-10 dataset \citep{krizhevsky2009learning}. In this setting the overparameterization is in terms of the expanded width of each hidden layer. In our experiments the expansion factor for overparameterization is set to $e=3$. We will compare the following algorithms,

\begin{itemize}
    \item \textbf{\textsc{Baseline}}. A standard feedforward network training.
    \item \textbf{\textsc{Majority Kernels}}. Training of an overparameterized network via MK with $e=3$ .
    \item \textbf{\textsc{Ensemble-Baseline}}. This baseline assesses true ensemble performance, setting the achievable performance ceiling. We train and ensemble three ($e=3$) independent models.
    \item    \textbf{\textsc{Distilled-Baseline}}. Assesses standard knowledge distillation to compress the model produced by the ensemble baseline to the original model architecture.
    \item \textbf{\textsc{Subset-Baseline}}. A baseline based on discrete optimization. The method treats model dimension reduction as a subset selection problem, commonly addressed through submodular maximization \citep{Nemhauser1978Submod,fujishige2005submodular}. This baseline is described in detail in Appendix~\ref{sec:subset}. Note that due to the invocation of a combinatorial procedure for each layer this baseline is computationally much more expensive and is impractical beyond simple architectures.

\end{itemize}
\vspace{0.25cm}

For each algorithm, we hypertune the learning rate by training with learning rates in $0.001 \times 1.5^i$ for $i \in [-4, 5]$ when the optimizer is the Adam optimizer. Similarly, when the optimizer is SGD we consider the range of learning rates to be in $0.025 \times 1.5^i$ for $i \in [-4, 5]$.
In each case we pick the best performing learning rate on a separate validation set. Finally, we report the test accuracy for the chosen learning rate.

We run the experiment on various architecture topologies:
\begin{itemize}
    \item $\mathcal{A}_1$ One hidden layer with 100 neurons.
    \item $\mathcal{A}_2$ Two hidden layers with $\{200, 100\}$ neurons.
    \item $\mathcal{A}_3$ Three hidden layers with $\{400, 200, 100\}$ neurons.
\end{itemize}

\noindent \textbf{Results:} The results are presented in Table~\ref{cifar10}. We see that across the three topologies, MK consistently outperforms the baseline and distilled baselines. Furthermore, it achieves performance comparable or better than the much more expensive subset selection based combinatorial approach. Our algorithm increases the training time on CPU by 16.67\% which is a negligible increase compared to the distilled baseline and the subset one.

\begin{table*}[ht]
\centering
\hspace{-11pt}\begin{minipage}{0.5\textwidth}
    \centering
    \begin{small}
    \begin{tabular}{cccc|}
    \toprule
\textbf{Model} & \textbf{Architecture} & \textbf{Optimizer} & \textbf{Test accuracy}\\
\midrule
$\mathcal{A}_1$ & [100] & Adam & $51.79 \pm 0.21$  \\
\rowcolor{evenlightergray} ensemble-$\mathcal{A}_1$ & [100] & Adam & $52.27 \pm 0.12$ \\
distilled-$\mathcal{A}_1$ & [100] & Adam & $48.68 \pm 0.86$  \\
Subset-$\mathcal{A}_1$ & [100] & Adam & $52.09 \pm 0.16$ \\
\rowcolor{lightergray} Majority-$\mathcal{A}_1$ & [100] & Adam & $52.01 \pm 0.10$   \\
\midrule
$\mathcal{A}_2$ & [200, 100] & Adam & $52.99 \pm 0.09$  \\
\rowcolor{evenlightergray} ensemble-$\mathcal{A}_2$ & [200, 100] & Adam &  $54.36 \pm 0.08$  \\
 distilled-$\mathcal{A}_2$ & [200, 100] & Adam &  $52.94 \pm 0.26$  \\
 Subset-$\mathcal{A}_2$ & [200, 100] & Adam &   $53.69 \pm 0.20$ \\
\rowcolor{lightergray} Majority-$\mathcal{A}_2$ & [200, 100] & Adam &  $54.29 \pm 0.20$ \\
\midrule
$\mathcal{A}_3$ & [400, 200, 100] & Adam & $53.79 \pm 0.18$ \\
\rowcolor{evenlightergray} ensemble-$\mathcal{A}_3$ & [400, 200, 100] & Adam &  $55.62 \pm 0.15$   \\
 distilled-$\mathcal{A}_3$ & [400, 200, 100] & Adam & $55.26 \pm 0.32$  \\
 Subset-$\mathcal{A}_3$ & [400, 200, 100] & Adam  & $54.58 \pm 0.15$  \\
\rowcolor{lightergray} Majority-$\mathcal{A}_3$ & [400, 200, 100] & Adam &  $55.04 \pm 0.12$ \\\bottomrule
    \end{tabular}
    \end{small}
\end{minipage}\hspace{28pt}
\begin{minipage}{0.39\textwidth}
    \centering
    \begin{small}
    \begin{tabular}{cccc}
    \toprule
     \textbf{Optimizer} & \textbf{Test accuracy}\\
    \midrule
 SGD &  $51.24 \pm 0.16$  \\
 \rowcolor{evenlightergray} SGD &   $52.55 \pm 0.52$ \\
 SGD &  $48.69 \pm 0.33$  \\
 SGD &  $52.10 \pm 0.20$  \\
\rowcolor{lightergray} SGD &  $51.82 \pm 0.15$ \\
 \midrule
 SGD &  $52.64 \pm 0.34$ \\
  \rowcolor{evenlightergray} SGD &  $54.87 \pm 0.15$  \\
 SGD &  $51.87 \pm 0.89$  \\
 SGD &  $53.90 \pm 0.15$ \\
 \rowcolor{lightergray} SGD &  $54.21 \pm 0.24$  \\
 \midrule
 SGD &  $53.30 \pm 0.17$  \\
  \rowcolor{evenlightergray} SGD &  $56.56 \pm 0.27$   \\
  SGD &  $54.21 \pm 0.31$ \\
 SGD  & $54.57 \pm 0.34$  \\
 \rowcolor{lightergray} SGD &  $54.97 \pm 0.31$  \\
\bottomrule
    \end{tabular}
    \end{small}
\end{minipage}
\caption{Results of the various algorithms trained on CIFAR-10 with batch size 256.}\vspace{-5pt}
\label{cifar10}
\end{table*}

\begin{figure}[ht]
\begin{minipage}{1\textwidth}
    \centering
    \includegraphics[scale=0.21]{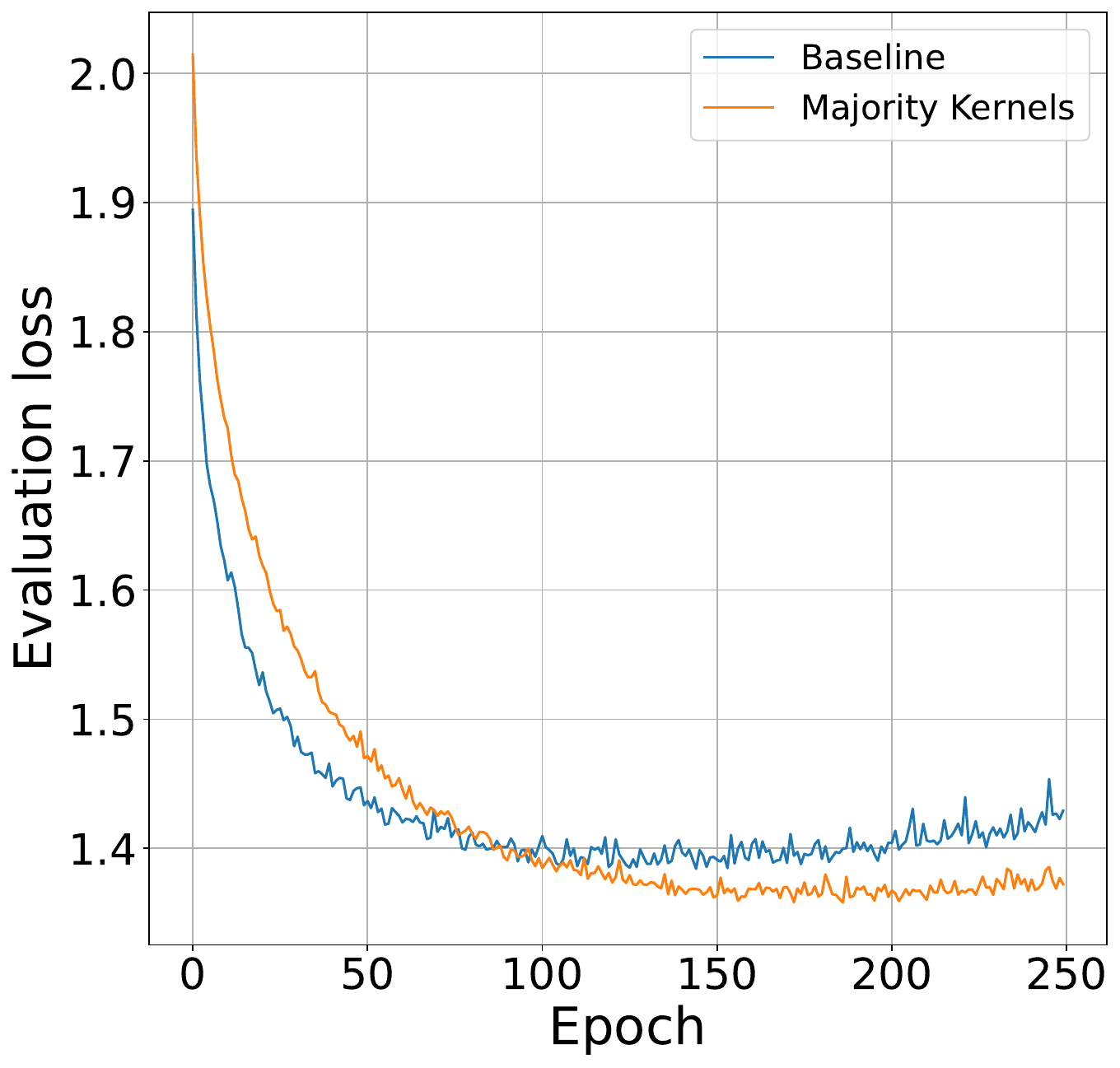}
    \caption{Test loss curves for training CIFAR-10 on Majority Kernels vs Baseline on the $\mathcal{A}_1$ architecture. \vspace{0.25cm}}
    \label{fig:dnn_training}
\end{minipage}\hspace{20pt}
\end{figure}

In the Figure~\ref{fig:dnn_training} below, we present the loss curves throughout training of the vanilla training vs the majority kernels on the $\mathcal{A}_1$ architecture. 
Notice that with the added regularization, MK often needs more steps to converge.

\end{document}